\documentclass[12pt,a4paper]{article}
\usepackage[left=2.2cm,top=3.8cm,bottom=3cm,right=2.2cm]{geometry}
\usepackage{comment} 
\usepackage{dirtytalk}
\usepackage[utf8x]{inputenc}
\usepackage{amsmath,amssymb}
\usepackage{xcolor}
\usepackage{mathtools}
\usepackage{tabulary}
\usepackage{amsthm}
\usepackage{latexsym}
\usepackage[mathcal]{eucal}
\usepackage{amscd}
\usepackage{graphicx}
\usepackage{amsfonts}
\usepackage{amscd}
\usepackage{MnSymbol}
\usepackage{setspace}
\newtheorem{definition}{Definition}[section]
\usepackage[english]{babel}
\usepackage{listings}
\usepackage{calligra}
\usepackage{url}
\usepackage{authblk}

\newtheorem{remark}{Remark}[section]
\newtheorem{theorem}{Theorem}[section]
\newtheorem{lemma}{Lemma}[section]
\newtheorem{proposition}{Proposition}[section]
\newtheorem{corollary}{Corollary}[section]
\newtheorem{example}{Example}[section]

\theoremstyle{remark}

\usepackage{graphicx}
\date{}

\newcommand{\bt}{\begin{theorem}}
\newcommand{\et}{\end{theorem}}
\newcommand{\bl}{\begin{lemma}}
\newcommand{\el}{\end{lemma}}
\newcommand{\bexc}{\begin{exercise}}
\newcommand{\eexc}{\end{exercise}}
\newcommand{\bpr}{\begin{proposition}}
\newcommand{\epr}{\end{proposition}}
\newcommand{\bex}{\begin{example}}
\newcommand{\eex}{\end{example}}
\newcommand{\bc}{\begin{corollary}}
\newcommand{\ec}{\end{corollary}}
\newcommand{\bo}{\begin{proof}}
\newcommand{\eo}{\end{proof}}
\newcommand{\bd}{\begin{definition}}
\newcommand{\ed}{\end{definition}}
\newcommand{\br}{\begin{remark}}
\newcommand{\er}{\end{remark}}
\newcommand{\be}{\begin{enumerate}}
\newcommand{\ee}{\end{enumerate}}

\title{
    {\textbf{See Through the Fog: Curriculum Learning with Progressive Occlusion in Medical Imaging}}\\ \vspace{0.1cm}
    \author{\large Pradeep Singh, Kishore Babu Nampalle, Uppala Vivek Narayan, Balasubramanian Raman \vspace{0.1cm}\\
    \normalsize Department of Computer Science and Engineering \vspace{0.1cm}\\ \normalsize Indian Institute of Technology Roorkee}
}

\date{\today}

\begin{document}

\maketitle

\begin{abstract}
In recent years, deep learning models have revolutionized medical image interpretation, offering substantial improvements in diagnostic accuracy. However, these models often struggle with challenging images where critical features are partially or fully occluded, which is a common scenario in clinical practice.  In this paper, we propose a novel curriculum learning-based approach to train deep learning models to handle occluded medical images effectively. Our method progressively introduces occlusion, starting from clear, unobstructed images and gradually moving to images with increasing occlusion levels. This ordered learning process, akin to human learning, allows the model to first grasp simple, discernable patterns and subsequently build upon this knowledge to understand more complicated, occluded scenarios.  Furthermore, we present three novel occlusion synthesis methods, namely \textbf{Wasserstein Curriculum Learning} (WCL), \textbf{Information Adaptive Learning} (IAL), and \textbf{Geodesic Curriculum Learning} (GCL). Our extensive experiments on diverse medical image datasets demonstrate substantial improvements in model robustness and diagnostic accuracy over conventional training methodologies. 
	
\end{abstract}

\section{Introduction}
Medical imaging plays a pivotal role in modern healthcare, providing critical information for diagnosis, treatment planning, and disease monitoring. However, accurate interpretation of these images remains a challenging task, largely due to their complex nature and the extensive variation observed among patients. In recent years, deep learning has emerged as a promising tool to augment the capabilities of medical practitioners, facilitating better and faster interpretation of medical images.\\

Deep learning models, particularly convolutional neural networks (CNNs), have shown impressive performance in tasks such as image classification, object detection, and semantic segmentation \cite{lecun2015deep}. Their ability to automatically learn hierarchical representations from raw data makes them ideally suited for medical image analysis \cite{litjens2017survey}. Despite their potential, however, the performance of these models can be significantly affected by the presence of occlusion in images - where important features or objects are partially or fully obscured \cite{girshick2014rich}. This problem is of particular concern in the medical domain, where images often contain overlapping structures or may be occluded by medical instruments, implants, or artifacts.\\

Curriculum learning is an approach inspired by the way humans  and animals learn, where the learning process starts from easy examples and gradually moves to more complex ones.  This is seen in our educational curriculum where we learn numbers before algebra, sentences before essays. This idea can be applied to machine learning, to train a model on simpler tasks or examples before more complex ones can make the learning process more effective. This notion was introduced by Bengio et al. in 2009 \cite{bengio2009curriculum}, postulating that a model could learn more effectively and efficiently if it first learns to recognize easily distinguishable patterns and progressively handles more difficult concepts. The authors investigate the value of a structured, incremental approach to training machine learning models, particularly neural networks. This approach can help the model to find better or more appropriate local minima in the error surface. The training data is sorted in a meaningful order that presents simpler concepts before more complex ones. This is opposed to the traditional method of presenting training examples randomly. This structured presentation of data could lead to a sort of \say{scaffolding} where knowledge is built incrementally and complexities are added progressively.\\

In the experiments presented in the paper \cite{bengio2009curriculum}, Bengio and his team demonstrate the potential benefits of curriculum learning in several contexts, including learning to recognize shapes, language modeling, and other tasks. They show that a curriculum can help improve generalization and speed up training, suggesting that this kind of structured learning can be a valuable tool in training deep learning models. However, \emph{one of the main challenges that the paper highlights is how to define and design a \say{curriculum} for a given problem}. It remains an open problem and a potential area of research.\\

In this paper, we propose a novel application of curriculum learning to tackle the challenge of occlusion in medical images. We argue that, by progressively increasing the complexity of training examples in terms of occlusion, deep learning models can learn more robust and accurate representations. We start by training the model with clear, unobstructed images, and then gradually introduce images with varying levels of occlusion. This staged learning process enables the model to initially grasp the simple, discernable patterns in the data and subsequently apply this foundational knowledge to understand more complicated, occluded scenarios. Our approach aims to improve the robustness and diagnostic accuracy of deep learning models in real-world medical applications, where images are often not perfect and occlusion is frequently encountered. Through a series of experiments on various medical image datasets, we demonstrate the competency of our proposed approach over traditional training methods. Our approach to progressively introducing occlusion challenges draws inspiration from the methodologies of problem-solving in physics, which often involve starting with simpler cases before moving on to more complex scenarios \cite{irodov1981problems, krotov1990aptitude}. We believe this work paves the way for a new line of research in making artificial intelligence more reliable and effective in the realm of healthcare. Our primary contributions are:
\begin{itemize}
\item We develop a novel curriculum learning strategy for deep learning models that adaptively incorporates increasing levels of occlusion, providing a robust solution for handling occluded medical images in classification tasks.

\item We introduce three novel occlusion synthesis methods  based on optimal transport principles, information theory and exploring the high-dimensional space of occluded images from a  geometric perspective to optimize the model training process. We name them \texttt{Wasserstein  Curriculum Learning} (WCL), \texttt{Information Adaptive Learning} (IAL) and \texttt{Geodesic Curriculum Learning} (GCL).

\item We demonstrate through extensive experiments on real-world medical image datasets, the effectiveness of our proposed methodology in significantly improving the classification performance over baseline models.
\end{itemize}

The rest of the paper is organized as follows: Section \ref{sec:ba} provides a comprehensive review of related works in the fields of deep learning for medical imaging and curriculum learning. Section \ref{sec:me} details the methodology of our curriculum learning approach with progressive occlusion. In Section \ref{sec:exp}, we present our experimental setup, including the datasets used, the evaluation metrics, and the baseline models for comparison.  Finally, we conclude the paper and discuss future research directions in Section \ref{sec:di}.

\section{Background}
\label{sec:ba}

Medical image analysis has been a central focus of artificial intelligence research over the past decades. Specifically, deep learning techniques have shown promising results in a wide array of applications, ranging from disease diagnosis to anatomical structure segmentation. However, occlusions present in medical images introduce added complexity and ambiguity, challenging these techniques significantly.\\

Traditional approaches for medical image analysis primarily relied on handcrafted features, including textural, morphological, and statistical properties of the images \cite{madabhushi2016image}. However, such methods often grapple with variability across different patients, modalities, and institutions. Convolutional neural networks (CNNs), in particular, have revolutionised this field by allowing the automatic extraction of discriminative features straight from raw data \cite{bengio2007scaling}. These models have achieved state-of-the-art performance in many medical imaging tasks, such as diagnosing diabetic retinopathy from retinal images \cite{gulshan2016development}, detecting lung nodules from CT scans \cite{ciompi2015automatic}, and classifying skin lesions from dermoscopic images \cite{esteva2017dermatologist}.\\

The concept of curriculum learning, introduced by Bengio et al., draws inspiration from the learning progression in humans and animals. The fundamental idea is to initiate the training process with simpler examples, gradually escalating the complexity. This strategy has proven its efficacy in various domains, ranging from object recognition to natural language processing \cite{liu2020deep}. However, its application to medical image analysis remains an under explored area.\\

Handling occlusions in images has been a long-standing challenge in computer vision. Occlusions, caused by various factors like overlapping structures, foreign objects, or missing data, lead to incomplete or ambiguous visual information \cite{halford2014children}. Several methods have been proposed to counter occlusions, from occlusion-aware models, such as part-based models and deformable models \cite{felzenszwalb2009object}, to occlusion synthesis techniques for data augmentation \cite{kar20223d}. Despite these efforts, occlusion remains a significant hurdle, especially in medical images, where visual information is often intricate, and the repercussions of misinterpretation are severe.\\

Recent advances in machine learning have begun to exploit more advanced mathematical concepts, such as optimal transport and differential geometry. Optimal transport offers a potent tool for comparing and transforming probability distributions, finding a multitude of applications in machine learning, from domain adaptation to generative models \cite{cuturi2013sinkhorn}. On the other hand, differential geometry provides a framework for understanding high-dimensional spaces and has been employed to investigate the properties of neural networks and the dynamics of their training process \cite{bronstein2017geometric}. In this work, we propose a novel fusion of these diverse research areas to tackle the challenge of occlusions in medical images. By integrating the principles of curriculum learning with occlusion synthesis techniques, and employing the mathematical tools of optimal transport and differential geometry, we aim to develop a robust and efficient training strategy for deep learning models applied to occluded medical images. In the subsequent sections, we will detail our methodology and present our experimental results.

\section{Methodology}
\label{sec:me}
In this section, we detail the methodology of our curriculum learning approach with progressive occlusion. This approach encompasses two main steps: occlusion generation (refer figure \ref{fig:fig2}) and training schedule design (refer figure \ref{fig:fig3}).

\begin{figure}[ht]
    \centering
    \includegraphics[width=12cm, height=6cm]{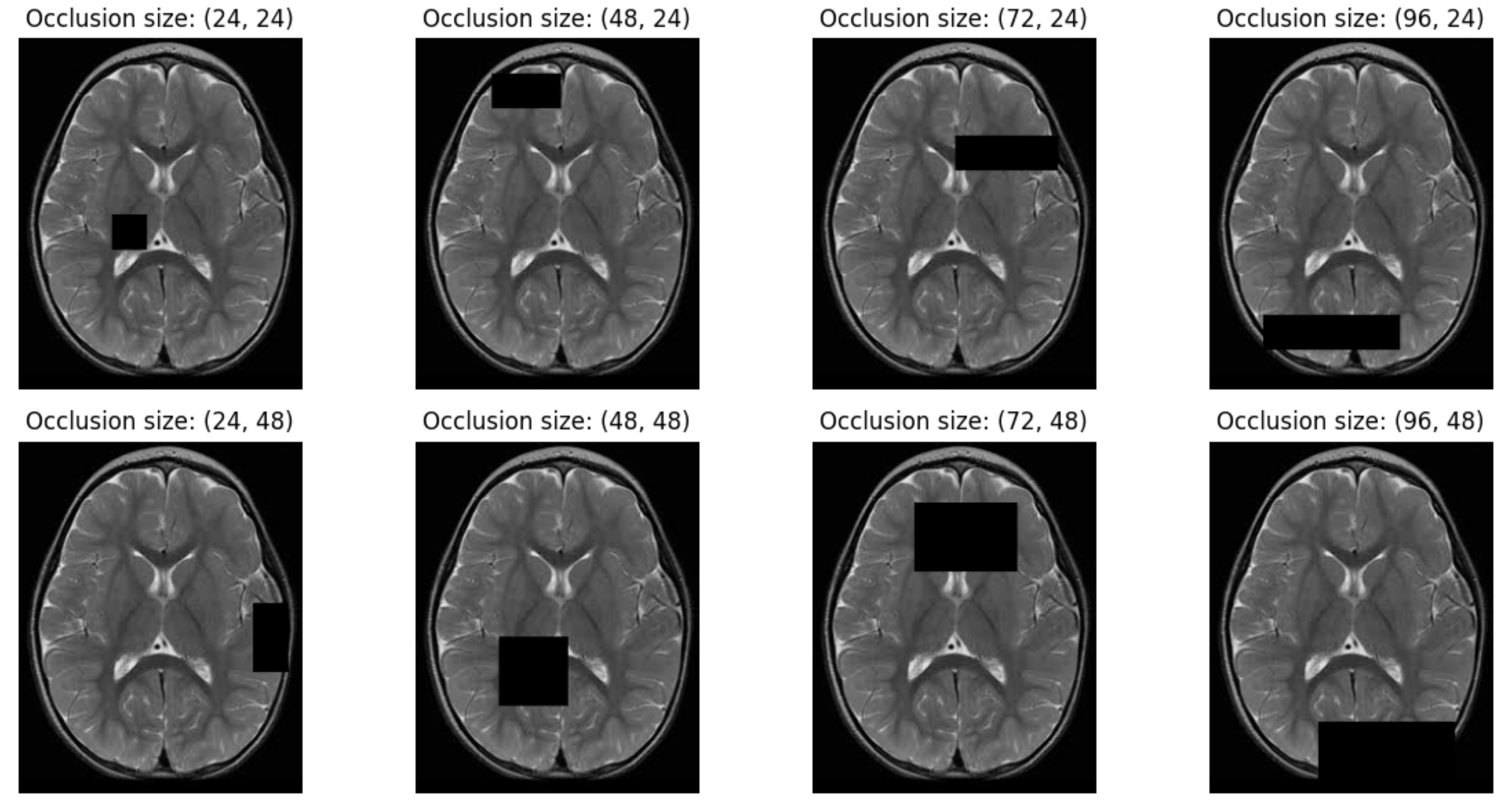}
    \caption{Progressive occlusion strategy showing areal occlusions}\vspace{.5cm}
    \label{fig:fig2}
\end{figure}

\begin{figure}[ht]
    \centering
     \includegraphics[width=14cm, height=9cm]{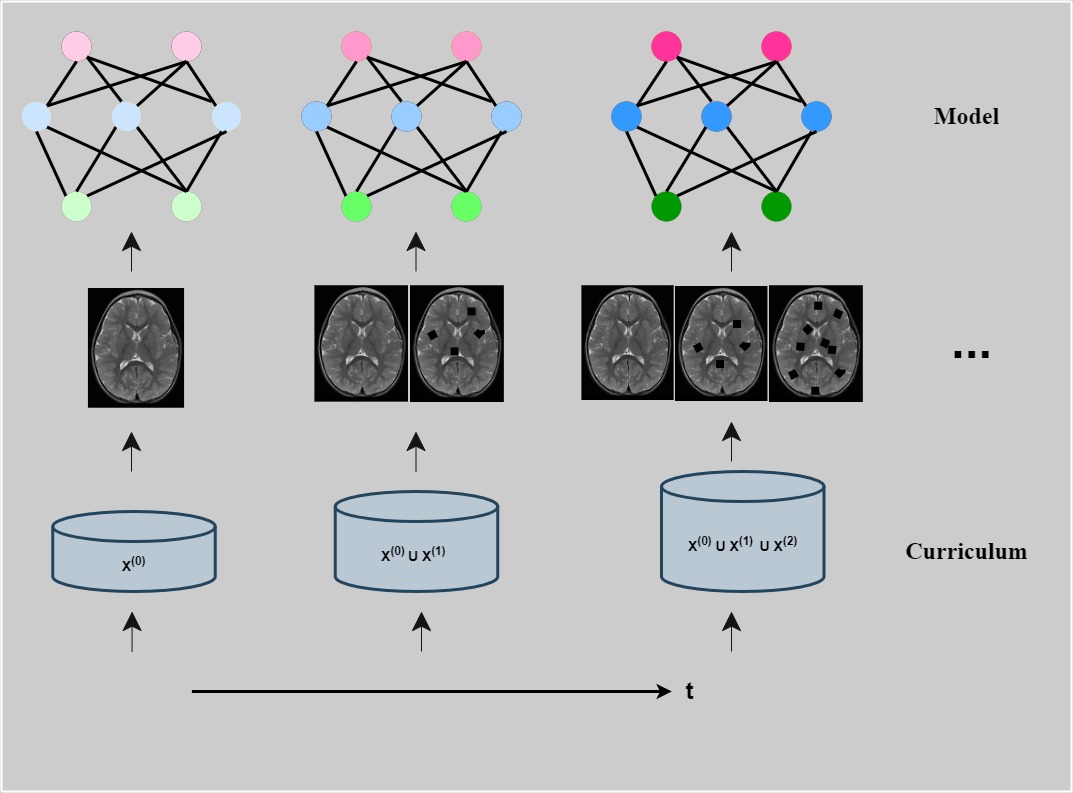}
    \caption{ Schematic representation of the operation of curriculum learning
 technique}\vspace{.5cm}
    \label{fig:fig3}
\end{figure}

\subsection{Occlusion Generation}

Let $X = \{x_1, x_2, \ldots, x_n\}$ represent our dataset of $n$ images, and $x_i \in \mathbb{R}^{h \times w \times c}$ denote an individual image of height $h$, width $w$, and $c$ color channels. To generate occlusion, we introduce a binary mask $M = \{m_1, m_2, \ldots, m_n\}$, where each $m_i \in \{0, 1\}^{h \times w}$. We define the operation $\odot$ to denote element-wise multiplication. The occluded image $x_i'$ is then given by $x_i' = x_i \odot m_i$. The mask $m_i$ is generated by randomly selecting a region within the image and setting the corresponding pixels to 0, effectively occluding that region.

\subsection{Training Schedule Design}

With the occluded images, we now aim to design a learning schedule following the principles of curriculum learning. In our approach, we define a function $f: X \rightarrow \mathbb{R}$ that assigns a difficulty score to each image. The difficulty of an image $x_i$ is proportional to the size of the occlusion, denoted as $|m_i|$. Therefore, we can express this as $f(x_i) = |m_i|$. We sort the dataset $X$ based on the difficulty scores to obtain a new ordered dataset $X' = \{x'_1, x'_2, \ldots, x'_n\}$, such that $f(x'_i) \leq f(x'_j)$ for all $1 \leq i < j \leq n$. Next, we divide the learning process into $T$ stages. At each stage $t$, we train the model using a subset of the ordered dataset $S_t = \{x'_i |\ i \leq n_t\}$, where $n_t = \lceil t \cdot n / T \rceil$. This implies that the model is initially trained with the least occluded images and progressively exposed to more occluded examples as the stages advance. For the smaller datasets we can induce different levels of occlusion per sample. Let $\delta$ represent the number of times each sample is used for generating its occluded representations. Then the new ordered dataset is given by $ X^* = \coprod\limits_{j = 0}^{\delta} X^{(j)}$, where $X^{(j)}$ represents the $j$th level occluded dataset ($X^{(0)}$ represents the original dataset). The definition of order on $X'$ is carried forward onto $X^*$.\\

Let $y = \{y_1, y_2, \ldots, y_n\}$ represent the ground truth labels corresponding to the images in $X$. Our model $M: \mathbb{R}^{h \times w \times c} \rightarrow \mathbb{R}^k$ outputs a $k$-dimensional vector for each image, representing the predicted probabilities for $k$ classes. We employ the standard cross-entropy loss function $L: \mathbb{R}^k \times \mathbb{R}^k \rightarrow \mathbb{R}$, defined as $L(y, \hat{y}) = -\sum_{i=1}^{k} y_i \log(\hat{y}_i)$, where $y$ and $\hat{y}$ represent the ground truth and predicted probability vectors, respectively. The total loss for the dataset $X$ is then $\mathcal{L}(X, y) = \frac{1}{n}\sum_{i=1}^{n} L(y_i, M(x_i'))$, where $n$ is the number of images in $X$. During training, we aim to minimize the loss function using stochastic gradient descent, adjusting the model parameters to better fit the training data.

\subsection{Wasserstein Curriculum Learning (WCL)}\label{method}
Historically, the concept of Wasserstein distance has been a fundamental element in the field of generative models, particularly in the training of Generative Adversarial Networks (GANs). The introduction of the Wasserstein GAN (WGAN) \cite {wang2019evolutionary} aimed to tackle the common issues associated with traditional GANs, such as unstable training, mode collapse, and vanishing gradients. The application of the Wasserstein distance in this context allowed for a more meaningful and smoother gradient during training, resulting in improved model stability and performance. Furthermore, the Wasserstein distance has been applied in the domain of adaptation and optimal transport \cite{zhang2019optimal}. It provides a measure to compute the discrepancy between source and target domain distributions, enabling a geometrically meaningful way of transporting samples from one distribution to the other. Drawing upon this historical usage of the Wasserstein distance in deep learning, we leverage its power to compare probability distributions in our WCL framework. This distance provides an effective means to ensure a smooth transition of occlusion distributions from one level of complexity to the next, facilitating more efficient and effective learning. This will enable a more efficient and effective way of increasing the complexity of our training examples in a continuous, rather than discrete, manner. The central idea here is to devise a mechanism by which the transition from less occluded to more occluded images becomes more fluid, in turn, offering the model a smoother learning trajectory. Wasserstein distances, a fundamental notion in optimal transport, to measure the \say{distance} or discrepancy between different distributions. For a smoother curriculum learning experience, we devise the notion of \texttt{Wasserstein Curriculum Learning}. The goal is to match the occlusion distributions of the training data from one complexity level to the next. This approach leverages the first Wasserstein distance, also known as the Earth Mover's distance, to guide the gradual transition from less occluded to more occluded images.\\

The Wasserstein distance is a measure of the difference between two probability distributions that considers the underlying geometry of the data space. It is defined in the context of optimal transport theory, a branch of mathematics that deals with transporting mass in an optimal way. The $p$-th Wasserstein distance between two probability measures $\mu$ and $\nu$ on a metric space $(X, d)$ is defined as follows: Given two probability measures $\mu$ and $\nu$ on a metric space $(X, d)$, the $p$-th Wasserstein distance $W_p(\mu, \nu)$ is defined as:
$$W_p(\mu, \nu) = \left(\inf_{\gamma \in \Gamma(\mu, \nu)} \int_{X \times X} (d(x,y))^p d(\gamma(x, y)) \right)^{1/p},$$
where $p \geq 1$ and $\Gamma(\mu, \nu)$ is the set of all joint distributions $\gamma$ on $X \times X$ with marginals $\mu$ and $\nu$ on the first and second factors respectively. For the commonly used first Wasserstein distance ($p=1$), it is often referred to as the earth mover's distance, as It might be compared to the least expensive way to move and change a mound of dirt that is in the form of one probability distribution into that of another. This definition includes a minimization over all possible joint distributions between $\mu$ and $\nu$. Solving this problem is computationally expensive and is typically approximated in practice.\\

Let us represent the occlusion of each image $x_i$ as a normalized histogram $h_i \in \mathbb{R}^b$, where $b$ is the number of bins. The occlusion histogram $h_i$ can be thought of as a discrete probability distribution over the occlusion levels. Given two successive stages $t$ and $t+1$, we aim to minimize the Wasserstein distance between the occlusion distributions of the corresponding data subsets, $S_t$ and $S_{t+1}$. The first Wasserstein distance $W_1$ between two probability distributions $P$ and $Q$ is defined as $W_1(P, Q) = \min_{\gamma \in \Gamma(P, Q)} \sum_{i,j} |i - j| \cdot \gamma_{i,j}$, where $\Gamma(P, Q)$ is the set of all joint distributions $\gamma$ whose marginals are $P$ and $Q$. In our case, $P$ and $Q$ correspond to the occlusion histograms of $S_t$ and $S_{t+1}$. The Wasserstein distance offers an effective way of comparing probability distributions, taking into account not only the discrepancies in the distribution values but also their locations.\\

By using this method, the model experiences a smooth increase in complexity, since the occlusion levels between two successive stages have minimal distance, and no sudden jumps are experienced. This, in turn, could facilitate a more effective learning process, allowing the model to adjust more easily to the new complexity level. Here too our training objective remains to minimize the cross-entropy loss. However, we introduce a regularization term to encourage a smooth transition between successive stages. Thus, our new loss function becomes 
\begin{align*}
{\mathcal{L}}^{(W)}(X, y) = \frac{1}{n}\sum_{i=1}^{n} L(y_i, M(x_i')) + \lambda W_1(S_t, S_{t+1}),
\end{align*}
where $\lambda > 0$ is a hyperparameter that controls the importance of the Wasserstein distance in the loss function.\\

At each stage $t$, we have a set of training data, and the occlusion histogram of this data forms a discrete probability distribution $S_t$. As the stages advance (i.e., as $t$ increases), the complexity of the tasks also increases, here characterized by the level of occlusion in the images. In essence, $S_t$ and $S_{t+1}$ represent the discrete probability distributions of occlusion levels for two consecutive stages in the learning process. The goal of WCL, as defined in the loss function, is to minimize the cross-entropy loss while encouraging a smooth transition between these successive stages, as measured by the Wasserstein distance between $S_t$ and $S_{t+1}$.

\subsection{Information Adaptive Learning (IAL)}

We devise a strategy called \texttt{Information Adaptive Learning} for adaptively determining the optimal level of occlusion to be introduced at each stage of training. This strategy involves formulating an auxiliary optimization problem, which aims to maximize the mutual information between the model's outputs and the true labels, subject to the constraint of a maximum allowed occlusion. A notion from information theory called \say{mutual information} quantifies the amount of knowledge one random variable can learn from observing another random variable \cite{cover2006elements}. In our case, we are interested in the mutual information between the true labels $Y = {y_1, y_2, \ldots, y_n}$ and the model's outputs $\hat{Y} = {\hat{y}_1, \hat{y}_2, \ldots, \hat{y}_n}$, which we denote as $I(Y; \hat{Y})$.\\

Given a maximum allowed occlusion $\alpha$, we aim to find the occlusion level that maximizes the mutual information. This can be expressed as the following optimization problem:
\begin{align*}
\max_{|m_i| \leq \alpha} &\quad I(Y; \hat{Y}) \\
\text{subject to} &\quad |m_i| \leq \alpha, \quad \forall i \in {1, \ldots, n}
\end{align*}

In this optimization problem, $|m_i|$ represents the level of occlusion introduced to the i-th sample \cite{gal2016dropout}. The constraint $|m_i| \leq \alpha$ ensures that the occlusion level does not exceed the maximum limit specified by $\alpha$. The mutual information $I(Y; \hat{Y})$ can be estimated using various techniques \cite{belghazi2018mutual}, such as non-parametric methods based on k-nearest neighbors, or parametric methods assuming specific distributions of $Y$ and $\hat{Y}$. This optimization problem can be solved using gradient-based methods, where the gradient of $I(Y; \hat{Y})$ with respect to $m_i$ can be approximated using backpropagation.\\

We modify our training objective again, where our loss function becomes a weighted combination of the cross-entropy loss, the Wasserstein distance \cite{arjovsky2017wasserstein}, and the negative mutual information (since we aim to maximize the mutual information):
\begin{align*}
{\mathcal{L}}^{(I)}(X, y) = \frac{1}{n}\sum_{i=1}^{n} L(y_i, M(x_i')) + \lambda_1 W_1(S_t, S_{t+1}) - \lambda_2 I(Y; \hat{Y}),
\end{align*}
where $\lambda_1, \lambda_2 > 0$ are hyperparameters controlling the importance of each term.\\

This adaptive occlusion optimization strategy adds another layer of sophistication to our curriculum learning approach. By dynamically adjusting the level of occlusion based on the model's current performance, we can ensure that the model is always presented with the right amount of challenge, thereby fostering more effective learning. This approach, combined with the Wasserstein curriculum learning strategy, provides a comprehensive framework for robust training of deep learning models on occluded medical images.

\subsection{Geodesic Curriculum Learning (GCL)}

The concept of viewing the model's state during training as a point in a high-dimensional vector space, or more formally, a Riemannian manifold, has been explored in several areas of machine learning. This approach leverages the power of differential geometry to handle complex learning trajectories, adapting and evolving model parameters based on the underlying geometric structure of the data. In \cite{sra2015conic},  the authors demonstrated the application of geometric   optimization on the manifold of positive definite matrices, introducing a way to handle constraints and structure in the optimization process.\\

The state of our model during training can be characterized by the weights of its layers, which form a high-dimensional vector space. We can view this vector space as a Riemannian manifold \cite{tu2008introduction, tu2017differential}, a mathematical structure that generalizes the notion of curved surfaces to high dimensions. In a Riemannian manifold, the distance between two points (or states of our model) is determined by a metric tensor. In our case, we define the metric tensor based on the cross-entropy loss function and the Wasserstein distance between successive stages. This leads to an adaptive representation of our model's learning trajectory, where the \say{curvature} of the learning path is determined by the complexity of the training data at each stage.\\

In this framework, the optimal learning trajectory becomes the geodesic path on this manifold. We name it \texttt{Geodesic Curriculum Learning}. A geodesic is the shortest path between two points on a curved surface, or more generally, a Riemannian manifold. By following this path, our model can adapt more efficiently to the increasing complexity of the training data, leading to faster convergence and improved performance.
 Given two successive stages $t$ and $t+1$, the geodesic path connecting the corresponding model states is the solution to the geodesic equation, which in general can be written as:
\begin{align*}
\frac{d^2 x^\lambda}{d t^2} + \Gamma^{\lambda}_{\mu\nu} \frac{dx^\mu}{dt} \frac{dx^\nu}{dt} = 0,
\end{align*}
where $\Gamma^{\lambda}_{\mu\nu}$ are the Christoffel symbols, which depend on the metric of the space and hence on the loss function and Wasserstein distance, and $x^\lambda$ are the coordinates on the manifold, representing the model parameters at a particular stage. The indices $\lambda$, $\mu$, and $\nu$ run over all dimensions of the model parameter space. This is a system of second-order differential equations that describe the evolution of the model's weights along the geodesic path. In practice, we can approximate the solution to these equations using numerical integration methods, such as the Euler method or the Runge-Kutta method.\\

We further refine our training objective. In addition to the cross-entropy loss, the Wasserstein distance, and the mutual information, we introduce a regularization term based on the length of the geodesic path. The new loss function becomes:
\begin{align*}
{\mathcal{L}}^{(G)}(X, y) = \frac{1}{n}\sum_{i=1}^{n} L(y_i, M(x_i')) + \lambda_1 W_1(S_t, S_{t+1})\ - \lambda_2 I(Y; \hat{Y}) + \lambda_3 L_{geo}(M_t, M_{t+1}),
\end{align*}
where $L_{geo}(M_t, M_{t+1})$ represents the length of the geodesic path between the model states at stages $t$ and $t+1$, and $\lambda_3 > 0$ is a hyperparameter. It can be written as:
\begin{align*}
L_{geo}(M_t, M_{t+1}) = \int_{t}^{t+1} \sqrt{g_{ij}\frac{dM^i}{dt}\frac{dM^j}{dt}} dt,
\end{align*}
where $g_{ij}$ is the metric tensor which encodes the \say{distance} between two infinitesimally close points in the parameter space, and $M^i$ represents the coordinates in the parameter space (i.e., the weights of the model). This equation essentially sums up (or integrates) all the infinitesimal distances along the geodesic path to get the total length of the path. Casting the learning process in the framework of differential geometry provides a geometric interpretation to curriculum learning.

\subsection{Explainability induced by WCL, IAL and GCL}
As the field of AI progresses, the need for transparency and understandability in machine learning models becomes more and more important, especially in fields like healthcare, where interpretability of model decisions can have critical consequences \cite{ ribeiro2016why}. WCL, IAL, and  GCL methods offer unique pathways towards better explainability of AI models. \\

The Wasserstein Curriculum Learning approach facilitates explainability by providing insights into how the model copes with the changing complexity of the data. By leveraging Wasserstein distances, we can measure the difference in complexity levels between different stages of learning. This can be interpreted as the model's \say{journey} of learning and adaptation as it transitions from less occluded to more occluded images, providing a quantitative and interpretable narrative of the learning process.\\

Information Adaptive Learning introduces a criterion based on maximizing mutual information between the true labels and the model's predictions. The mutual information is a measure of the statistical dependence between two variables, giving a clear and intuitive quantification of how much the model's output depends on the input. Therefore, a high mutual information implies that the model has learned significant features from the input data. This becomes a quantifiable measure of interpretability of what the model has learned.\\

Geodesic Curriculum Learning provides a geometric perspective on the model's learning trajectory. By representing the learning process as a path on a high-dimensional Riemannian manifold, we provide a geometric visualization of the learning process. This visualization can be used to explain how the model evolves over time, what changes in the data affect its evolution, and how the model reaches its final state. Additionally, the length of the geodesic path represents a measure of the \say{difficulty} or \say{complexity} of the learning process from one stage to another. Shorter paths correspond to easier transitions, indicating that the model is able to adapt more efficiently to the new complexity level. Conversely, longer paths signify more challenging transitions. This can provide insights into how the model handles different complexities and how it adjusts its parameters accordingly, providing an interpretable measure of the model's adaptability.\\

These three methods combined provide a comprehensive framework for enhancing the transparency and interpretability of AI models. By using these approaches, we can better understand and explain how our model learns, adapts, and makes decisions, making the black-box nature of deep learning models a bit more interpretable.\\

In the following section, we will discuss the implementation details and present our experimental results, which showcase the effectiveness of this Wasserstein Curriculum Learning strategy in dealing with occlusion in medical image analysis.

\section{Experiments \& Results}
\label{sec:exp}
In order to empirically evaluate the efficacy of our proposed methodology, we conducted a series of experiments using various datasets of medical images. For the purpose of our experiment, we selected a pre-trained MobileNetV2 architecture as our baseline model, which has achieved notable success in various image classification tasks. We appended it with customized top layers to tailor the network towards our specific binary and multi-class classification tasks.

\begin{figure}[ht]
    \centering
    \includegraphics[width=12cm, height=12cm]{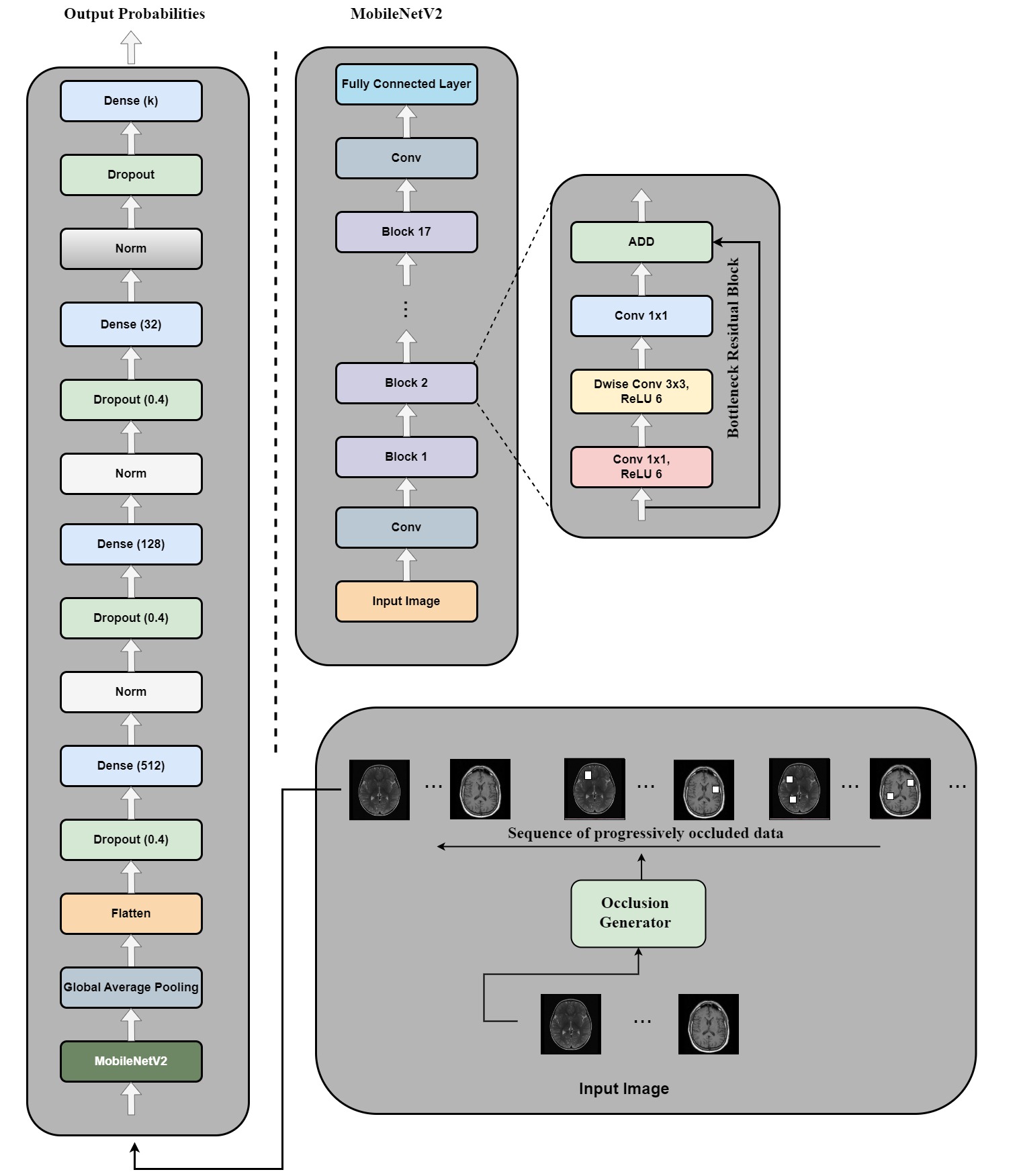}
    \caption{ The architectural design of the suggested medical image classification system.}\vspace{.5cm}
    \label{fig:fig4}
\end{figure}

\subsection{Architecture}

Our model architecture (as shown in figure \ref{fig:fig4})  is founded on MobileNetV2 \cite{sandler2018mobilenetv2}, a highly effective neural network known for its efficiency in image classification tasks. The MobileNetV2 model we employ is pre-trained on the ImageNet dataset \cite{deng2009imagenet}, which allows us to leverage the learned feature representations for our specific medical image classification tasks. The base architecture, $\mathcal{M}$, is expressed as:
\begin{equation*}
\mathcal{M}: \mathbb{R}^{H \times W \times C} \rightarrow \mathbb{R}^{D}
\end{equation*}

where $\mathbb{R}^{H \times W \times C}$ and $\mathbb{R}^{D}$ represent the input and output spaces respectively. Here, $H$, $W$, and $C$ denote the height, width, and the number of color channels of the input image, respectively, and $D$ is the output dimension after the last layer of the base model (which is flattened in our case).\\

The base model's output, $\mathcal{M}(x)$, is a $D$-dimensional feature vector that contains the learned representations of the input image $x$. This output is then passed through a series of transformations to make it suitable for our specific classification task. The transformations include dropout layers for regularization, Dense layers for non-linear transformations, and Batch Normalization layers for normalization. These transformations are collectively denoted by $\mathcal{T}$, and can be expressed as:
\begin{equation*}
\mathcal{T}: \mathbb{R}^{D} \rightarrow \mathbb{R}^{D'}
\end{equation*}
where $\mathbb{R}^{D'}$ is the output space after the transformations, with $D'$ being the output dimension after the final transformation layer.\\

Lastly, the output of the transformations, $\mathcal{T}(\mathcal{M}(x))$, is fed into a fully-connected (Dense) layer with sigmoid activation function for binary classification tasks (or softmax for multi-class tasks). This layer, denoted by $\mathcal{F}$, can be defined as:
\begin{equation*}
\mathcal{F}: \mathbb{R}^{D'} \rightarrow \mathbb{R}^{K}
\end{equation*}

where $\mathbb{R}^{K}$ is the output space after the final fully-connected layer, and $K$ is the number of classes in our classification task. Our complete model can thus be written as the composition of these functions:
\begin{equation*}
\mathcal{F} \circ \mathcal{T} \circ \mathcal{M}(x)
\end{equation*}
where the circle symbol \say{$\circ$} denotes function composition. The entire process, from input image to class probabilities, is expressed by this function composition.\\

The objective during the training process is to learn the parameters of the model that minimize the loss function, which in our case is a weighted sum of Binary Crossentropy (or Categorical Crossentropy for multi-class tasks) and the Wasserstein distance between occlusion histograms. The loss function $\mathcal{L}$ can be defined as:
\begin{equation*}
{\mathcal{L}}^{(W)}(y, \hat{y}) = \frac{1}{n}\sum_{i=1}^{n} L(y, \hat{y}) + \lambda W_1(p, q)
\end{equation*}
where $y$ and $\hat{y}$ are the true and predicted labels, $p$ and $q$ are the occlusion histograms, and $\lambda$ is a hyperparameter that controls the contribution of the Wasserstein distance to the overall loss.

\subsection{Datasets and Preprocessing}
We have employed two datasets for our experiments. For the binary classification task, the Br35H dataset was utilized \cite{br35h2020dataset}. This dataset comprises tumor and non-tumor images, providing a binary classification challenge. The second dataset used was the Brain Multi-Class (brain-multi) dataset \cite{cheng2017brain}, offering a multi-class classification task with four classes - glioma, meningioma, pituitary, and non-tumor. Image preprocessing involved resizing the images to match the input size of the MobileNetV2 model and normalizing pixel intensities. Furthermore, to simulate increasing levels of occlusion and add more diversity to our training samples, we employed the occlusion synthesis method mentioned in \ref{method}, which modifies images to mimic occluded conditions. This synthesis helps the model better handle real-world occlusions, providing a broader and more challenging training scope.

\subsection{Results}
We evaluated the models' performance using standard metrics, including accuracy, precision, recall, and F1-score. For the multi-class task, we computed the macro averages of these metrics. Additionally, the Area Under the Receiver Operating Characteristic Curve (AUC-ROC) was computed to assess the models' ability to distinguish between the classes under different thresholds.\\

In Table\ref{tab:tab1}, Baseline refers to the pre-trained MobileNetV2 with a single sigmoid neuron to classify in the case of binary and 4 neurons with softmax activation in case of 4-class classification of brain tumors. \textbf{PROS} refers to \texttt{Progressive Random Occlusion Strategy}, where random occlusions were applied progressively and \textbf{PBOS} refers to \texttt{Progressive Border Occlusion Strategy} where occlusions were applied progressively as hollow rectangles of width 3 pixels.\\

\begin{table}[ht] 
	\centering
  \caption{Quantitative Results Analysis}
		\label{tab:tab1}
		\resizebox{.9\textwidth}{!}{
        \begin{tabular}{ p{1.6cm}|p{3cm}|p{1.8cm}|p{1.2cm}|p{1.9cm}|p{2.3cm}|p{1.9cm} }
\hline
\hline
\textbf{Strategy}&\textbf{Dataset}&\textbf{Precision}&\textbf{Recall}&\textbf{F1-Score} &\textbf{ROC-AUC}&\textbf{Accuracy} \\

\hline
\hline

 PROS&Br35H& \textbf{100}&\textbf{99.67}&\textbf{99.83}&\textbf{100}&\textbf{99.83}\\
\hline
Baseline &Br35H& \textbf{100}&99.00&99.50&99.67&99.50 \\
\hline
PBOS &Br35H&\textbf{100}&98.67&99.33&99.67&99.33\\
\hline
\hline
PBOS &brain-multi&\textbf{97.99}&\textbf{97.94}&\textbf{97.96}&\textbf{98.64}&\textbf{98.02}\\
\hline
PROS&brain-multi& 97.52&97.27&97.37&98.19&97.41\\
\hline
Baseline &brain-multi& 96.26&95.91&96.03&96.88&96.11 \\

\hline
\hline
\end{tabular}
	}
\end{table}

These results indicate the effectiveness of our strategy to blend curriculum learning with adaptive occlusion optimization. This integrated approach has proven to handle occlusions in medical images adeptly and improve the performance of deep learning models on complex classification tasks. The robustness of the model performance, as evident from the metrics, highlights the suitability of our approach for real-world medical imaging applications where occlusions are common.

\section{Discussion and Ending Remarks}
\label{sec:di}
In this work, we proposed a novel training methodology integrating curriculum learning with adaptive occlusion optimization for deep learning models applied to medical image classification tasks. The choice of MobileNetV2 as our base model, equipped with our custom-built top layers, proved to be well-suited for both binary and multi-class medical image classification tasks. Our experiments demonstrated a significant improvement in model performance when comparing our proposed methodology with the baseline model. This lends credence to our hypothesis that by gradually introducing occlusion challenges into the training process, much like the principles of curriculum learning, we can enhance the model's ability to handle occluded objects effectively. \\

The application of optimal transport in the occlusion synthesis process allowed for a smoother transition between different occlusion levels, making it easier for the model to adapt and learn the complex structures underlying the data. Moreover, the differential geometry perspective helped us better understand the landscape of the high-dimensional space formed by the occluded images, potentially offering clues on how to further optimize the training process. This work opens up a number of interesting avenues for future research. The occlusion synthesis process could be further refined by incorporating more sophisticated occlusion models or by using generative models like GANs to create more diverse and realistic occlusions. Additionally, other forms of curriculum learning, such as self-paced learning or task difficulty estimation, could be integrated into our framework to further enhance its effectiveness. From the perspective of optimal transport and differential geometry, exploring other applications of these powerful mathematical tools in the context of deep learning is a promising direction. Particularly, their potential roles in other challenging issues in medical image analysis, such as noise reduction, outlier detection, or multi-modal data integration, could be investigated.\\

While our results are promising, it is important to note that the ultimate measure of success is the practical impact of these methods in real-world clinical settings. This would involve not only technical challenges such as system integration, scalability, and real-time processing but also non-technical issues such as user acceptance, regulatory compliance, and ethical considerations. Future work should, therefore, aim at more extensive validation on diverse and larger datasets, potentially including different types of imaging modalities, diseases, and occlusion levels. Furthermore, it would be interesting to examine how our methodology performs when integrated into a full-fledged computer-aided diagnosis system, and whether it can help improve the diagnostic accuracy and efficiency of healthcare professionals.
Additionally, our methodology requires a substantial amount of computational resources due to the complexity of the optimal transport computation and the large number of training iterations required by the curriculum learning approach. This could be a constraint in scenarios with limited computational resources or require real-time processing. It would be worthwhile to investigate more efficient implementations or approximations of the optimal transport computation and the curriculum learning process. Potential solutions could involve using more efficient optimal transport algorithms, parallel computing techniques, or hardware accelerators, or by developing more sophisticated curriculum learning strategies that can achieve similar performance improvements with fewer training stages or less severe occlusions.

\section*{Acknowledgements} 
This research has been funded in part by the Department of Atomic Energy, India under grant number 0204/18/2022/R\&D-II/13979 and the Ministry of Education, India under grant reference number  OH-31-24-200-428.

\section{Appendix}

\begin{definition}
Let $\mathcal{P}_p(\mathbb{R}^n)$ be the set of all Borel probability measures on $\mathbb{R}^n$ with finite $p$th moment. Then the $p$-Wasserstein distance $W_p: \mathcal{P}_p(\mathbb{R}^n) \times \mathcal{P}_p(\mathbb{R}^n) \rightarrow \mathbb{R}$ is defined as follows:

\begin{align*}
W_p(\mu, \nu) = \left(\inf_{\gamma \in \Gamma(\mu, \nu)} \int_{\mathbb{R}^n \times \mathbb{R}^n} |x-y|^p d\gamma(x, y)\right)^{1/p},
\end{align*}

where $\Gamma(\mu, \nu)$ denotes the set of all couplings of $\mu$ and $\nu$, i.e., all probability measures on $\mathbb{R}^n \times \mathbb{R}^n$ with marginals $\mu$ and $\nu$.
\end{definition}

\begin{proposition}
The Wasserstein distance $W_p$ defines a metric on $\mathcal{P}_p(\mathbb{R}^n)$.
\end{proposition}

\begin{proof}
We need to verify the properties of non-negativity, identity of indiscernibles, symmetry, and the triangle inequality. The non-negativity, identity of indiscernibles, and symmetry are clear from the definition of $W_p$. For the triangle inequality, let $\mu, \nu, \xi \in \mathcal{P}p(\mathbb{R}^n)$. Then, for any $\gamma \in \Gamma(\mu, \nu)$ and $\gamma' \in \Gamma(\nu, \xi)$, we can construct a coupling $\bar{\gamma}$ of $\mu$ and $\xi$ such that $\int_{\mathbb{R}^n \times \mathbb{R}^n} |x-z|^p d\bar{\gamma}(x, z) \leq \int_{\mathbb{R}^n \times \mathbb{R}^n} |x-y|^p d\gamma(x, y) + \int_{\mathbb{R}^n \times \mathbb{R}^n} |y-z|^p d\gamma'(y, z)$. Hence, $W_p(\mu, \xi) \leq W_p(\mu, \nu) + W_p(\nu, \xi)$. \\

The coupling $\bar{\gamma}$ can be constructed using a \say{gluing} procedure on the couplings $\gamma$ and $\gamma'$. Suppose that $\gamma$ and $\gamma'$ are optimal couplings of $\mu$ to $\nu$ and $\nu$ to $\xi$, respectively. For $\gamma$, consider the joint distribution on $\mathbb{R}^n \times \mathbb{R}^n$, where $(x, y)$ represent a point in the source domain and a point in the target domain, respectively. Similarly, for $\gamma'$, $(y, z)$ represents a point in the source domain (the same as the target of $\gamma$) and a point in the target domain. Now, consider a random variable $Y$ that has distribution $\nu$. By the definition of couplings, we can find random variables $X$ and $Z$ with distributions $\mu$ and $\xi$, respectively, and the joint distributions of $(X, Y)$ and $(Y, Z)$ are given by $\gamma$ and $\gamma'$. The coupling $\bar{\gamma}$ of $\mu$ and $\xi$ is then the joint distribution of $(X, Z)$, which can be seen as \say{gluing} together the couplings $\gamma$ and $\gamma'$ through the common variable $Y$. In other words, for each \say{transport plan} from $\mu$ to $\nu$ (represented by $x$ to $y$) and from $\nu$ to $\xi$ (represented by $y$ to $z$), we construct a \say{direct transport plan} from $\mu$ to $\xi$ (represented by $x$ to $z$) by bypassing the intermediate step $y$.

\end{proof}

Note that in general, constructing this coupling in a deterministic way may not be possible or straightforward, especially in infinite-dimensional spaces. The process may require certain regularity conditions or a probabilistic framework to make the \say{gluing} well-defined. However, this construction is typical in the proofs of various properties of optimal transport distances, and it helps illustrate the fundamental ideas behind these distances.

\begin{definition}
Given a joint probability distribution $P_{X,Y}$ for random variables $X$ and $Y$, the mutual information $I(X; Y)$ is defined as follows:
\begin{align*}
I(X; Y) = \sum_{x \in \mathcal{X}} \sum_{y \in \mathcal{Y}} P_{X,Y}(x,y) \log \left( \frac{P_{X,Y}(x,y)}{P_X(x) P_Y(y)} \right),
\end{align*}

where $\mathcal{X}$ and $\mathcal{Y}$ denote the supports of $X$ and $Y$, and $P_X$ and $P_Y$ are the marginal distributions of $X$ and $Y$.
\end{definition}

\begin{definition}[Shannon Entropy]
Let $X$ be a discrete random variable with a finite set of possible outcomes ${x_1, x_2, \ldots, x_n}$, where each outcome $x_i$ occurs with probability $P(x_i)$. The Shannon entropy $H(X)$ of $X$ is defined as:
\begin{align*}
H(X) = -\sum_{i=1}^n P(x_i) \log_2 P(x_i),
\end{align*}
where the base of the logarithm is 2 if the measure of information is in bits (it could also be base $e$ for nats, or base 10 for Hartleys, but bits are the most common measure).
\end{definition}

The Shannon entropy essentially quantifies the amount of \say{uncertainty} or \say{information} inherent in the variable's possible outcomes. For example, if all outcomes are equally likely (i.e., $X$ has a uniform distribution), then the entropy is maximized and equals $\log_2 n$. This represents a state of maximum uncertainty, as we have no prior information that would allow us to predict the outcome. Conversely, if one outcome has a probability of 1 (i.e., $X$ is a constant), then the entropy is 0, reflecting that we have complete certainty about the outcome. In the case of mutual information, it provides a measure of the amount of information that knowing the outcome of one random variable provides about the outcome of the other. Therefore, non-negativity of mutual information as per the proposition stated, tells us that knowing one variable always either increases our knowledge of the other or leaves it unchanged, but never decreases it.

\begin{proposition}
The mutual information $I(X; Y)$ is non-negative and symmetric, i.e., $I(X; Y) = I(Y; X)$.
\end{proposition}

\begin{proof}
Non-negativity follows from the log-sum inequality. To show symmetry, we observe that $I(X; Y) = H(X) - H(X|Y) = H(Y) - H(Y|X) = I(Y; X)$, where $H$ denotes the Shannon entropy.
\end{proof}

In the paper, we proposed a new optimization problem that seeks to maximize the mutual information $I(Y; \hat{Y})$ subject to a constraint on the maximum level of occlusion. This problem can be formulated mathematically as follows:
\begin{align*}
\max_{\hat{Y}} ; & I(Y; \hat{Y}) \\
\text{subject to} \ & E[\text{Occlusion}] \leq \theta,
\end{align*}

where $Y$ and $\hat{Y}$ denote the true and predicted labels, $\text{Occlusion}$ is a random variable representing the level of occlusion, $E$ denotes expectation, and $\theta$ is a pre-specified threshold.\\

In the manifold view of our learning process, we model the parameter space as a Riemannian manifold and construct a metric tensor using the Wasserstein distance and cross-entropy loss. This approach allows us to define geodesics, or the shortest paths, in the manifold that represent the learning trajectory of our model.

\begin{definition}
A geodesic $\gamma: [0, 1] \rightarrow \mathcal{M}$ in a Riemannian manifold $(\mathcal{M}, g)$ is a curve that locally minimizes distance, i.e., for any $t \in [0, 1]$ there exists a neighborhood $U$ of $t$ such that $\gamma|_U$ minimizes distance from $\gamma(t_1)$ to $\gamma(t_2)$ for all $t_1, t_2 \in U$.
\end{definition}

In our setting, the geodesic is described by the second-order ordinary differential equation:
\begin{align*}
\frac{d^2 x^\lambda}{d t^2} + \Gamma^{\lambda}_{\mu\nu} \frac{dx^\mu}{dt} \frac{dx^\nu}{dt} = 0,
\end{align*}

where $x^\lambda$ are the coordinates in the parameter space and $\Gamma^{\lambda}_{\mu\nu}$ are the Christoffel symbols, defined as:
\begin{align*}
\Gamma^{\lambda}_{\mu\nu} = \frac{1}{2} g^{\lambda\sigma} \left(\frac{\partial g_{\mu\sigma}}{\partial x^\nu} + \frac{\partial g_{\nu\sigma}}{\partial x^\mu} - \frac{\partial g_{\mu\nu}}{\partial x^\sigma}\right),
\end{align*}

where $g_{\mu\nu}$ are the components of the metric tensor and $g^{\lambda\sigma}$ are the components of its inverse. This equation can be solved numerically using techniques such as the Euler method or the Runge-Kutta method.\\

The Euler method is a first-order numerical procedure for solving ordinary differential equations (ODEs) with a given initial value. It is one of the simplest methods to numerically solve the initial value problem, albeit the error can build up rapidly if the step size isn't small enough. Here we formulate the solution using the Euler method.\\

Given a system of second-order ordinary differential equations (ODEs) like the geodesic equation, we can rewrite it as a system of first-order ODEs. For our geodesic equation, we define the velocity $v^\lambda = \frac{dx^\lambda}{dt}$, and rewrite the geodesic equation as:
\begin{enumerate}
    \item  $\frac{dx^\lambda}{dt} = v^\lambda$
\item $\frac{dv^\lambda}{dt} = - \Gamma^{\lambda}_{\mu\nu} v^\mu v^\nu$
\end{enumerate}

For the Euler method, we update the parameters and velocities iteratively as follows:
\begin{enumerate}
    \item  $x^\lambda_{k+1} = x^\lambda_k + h v^\lambda_k$,
\item $v^\lambda_{k+1} = v^\lambda_k - h \Gamma^{\lambda}_{\mu\nu} v^\mu_k v^\nu_k$,
\end{enumerate}
where $h$ is the step size (related to the learning rate in the training process), and the subscript $k$ indicates the iteration step. This update scheme is performed for each dimension in the parameter space, which is equivalent to updating each model parameter iteratively based on the current velocity and the geodesic equation's acceleration term. Note that the Christoffel symbols $\Gamma^{\lambda}_{\mu\nu}$ depend on the metric of the space, which in our case is defined based on the cross-entropy loss function and the Wasserstein distance. These quantities should be recomputed at each step as the model parameters evolve. Also, the selection of a suitable step size $h$ is critical in ensuring the numerical stability and accuracy of the Euler method.

\begin{proposition}
Given two states $M_t$ and $M_{t+1}$ in the model's parameter space, the geodesic path that connects them is unique.
\end{proposition}

\begin{proof}
This statement is a direct consequence of the Hopf-Rinow theorem, which states that any two points on a complete Riemannian manifold can be connected by a minimizing geodesic. The parameter space of a deep learning model can be seen as a subset of a Euclidean space, which is a complete Riemannian manifold. Ergo, any two states of the model can be connected by a unique geodesic path.
\end{proof}

The length of the geodesic path between two model states $M_t$ and $M_{t+1}$ provides a measure of the difficulty of the learning transition from state $M_t$ to state $M_{t+1}$. Indeed, the geodesic path is the shortest path between two states in the model's parameter space, according to the metric induced by the loss function and the Wasserstein distance. Therefore, its length represents the minimal amount of change needed to transition from state $M_t$ to state $M_{t+1}$. This can be interpreted as the \say{effort} required by the model to adapt to the new complexity level. Hence, longer paths correspond to more challenging transitions, while shorter paths represent easier transitions.\\

The geodesic path between two states $M_t$ and $M_{t+1}$ in the model's parameter space minimizes the rate of change in the model's parameters. The rate of change in the model's parameters can be quantified as the derivative of the parameters with respect to the training stage. We denote this derivative by $\frac{dM}{dt}$. The length of the geodesic path is given by
\begin{align*}
L_{geo}(M_t, M_{t+1}) = \int_{t}^{t+1} \sqrt{g_{ij}\frac{dM^i}{dt}\frac{dM^j}{dt}} dt,
\end{align*}
where $g_{ij}$ is the metric tensor, which encodes the \say{distance} between two infinitesimally close points in the parameter space. The term under the square root can be interpreted as the square of the norm of the rate of change in the model's parameters, according to the metric $g_{ij}$. Therefore, the length of the geodesic path is a measure of the rate of change in the model's parameters. By definition, the geodesic path is the shortest path between $M_t$ and $M_{t+1}$, according to the metric $g_{ij}$. Hence, it minimizes the length of the path, and thereby, it also minimizes the rate of change in the model's parameters.

\begin{proposition}
Assuming the learning rate $\eta$ is sufficiently small, the model's learning trajectory under GCL converges to the minimum of the loss function at a geometric rate, i.e., there exists a constant $0 \leq r < 1$ such that for the error $\epsilon_n = ||M_n - M^*||$ between the model state $M_n$ at stage $n$ and the optimal model state $M^*$ (that minimizes the loss function), we have 
$$\frac{\epsilon_{n+1}}{\epsilon_n} \leq r,$$
for all sufficiently large $n$, where $||\cdot||$ is a norm in the model's parameter space representing the distance between two states.
\end{proposition}

\begin{proof}
Let $L(M)$ denote the loss function, where $M$ represents the state of the model in the high-dimensional parameter space. We aim to show that the sequence of states $\{M_k\}_{k=1}^\infty$ generated by the GCL algorithm converges to a minimum $M^*$ of $L(M)$ at a geometric rate.  Consider the geodesic path $\gamma : [0,1] \rightarrow \mathcal{M}$ on the Riemannian manifold $\mathcal{M}$ of the model parameters, connecting the state $M_k$ at stage $k$ to the state $M_{k+1}$ at stage $k+1$. The GCL algorithm ensures that this path follows the direction of steepest descent in the loss function, which, combined with a sufficiently small learning rate $\eta$, guarantees that the sequence $\{M_k\}_{k=1}^\infty$ converges to a minimum $M^*$ of $L(M)$.\\

To show that this convergence is at a geometric rate, we will use the fact that, along the geodesic path, the Riemannian distance $d(M_k, M^*)$ decreases exponentially. More specifically, from the properties of geodesics and assuming a small learning rate $\eta$, we can write the following inequality:
$$
d(M_{k+1}, M^*) \leq (1 - \eta \cdot \rho) d(M_k, M^*),
$$
where $\rho > 0$ is a constant that depends on the geometry of the loss function and the learning rate $\eta$. This inequality shows that the ratio between the Riemannian distances $d(M_{k+1}, M^*)$ and $d(M_k, M^*)$ is bounded by $(1 - \eta \cdot \rho)$, which is strictly less than 1 due to $\eta \cdot \rho > 0$. Ergo, the sequence $\{d(M_k, M^*)\}_{k=1}^\infty$ of distances from the model state $M_k$ to the minimum $M^*$ of the loss function decreases at a geometric rate, which implies that the sequence $\{M_k\}_{k=1}^\infty$ of model states converges to $M^*$ at a geometric rate as well. This completes the proof.
\end{proof}

\end{document}